%% file: main.tex
\documentclass[letterpaper]{article} 
\usepackage{aaai24}  
\usepackage{times}  
\usepackage{helvet}  
\usepackage{courier}  
\usepackage[hyphens]{url}  
\usepackage{graphicx} 
\def\UrlFont{\rm}  
\usepackage{natbib}  
\usepackage{caption} 
\usepackage{algorithm}
\usepackage{algorithmic}
\usepackage{xcolor} 
\usepackage{amsthm}
\usepackage{pifont}
\newcommand{\xmark}{\ding{55}}%
\usepackage{amsmath}
\usepackage{amssymb}
\usepackage{bm}
\newcommand{\indep}{\perp \!\!\! \perp}

\newcommand*\diff{\mathop{}\!\mathrm{d}}

\newtheorem{theorem}{Theorem}

\newtheorem{lemma}{Lemma}

\newtheorem{definition}{Definition}

\usepackage{caption}
\usepackage{subcaption}
\usepackage{graphicx}

\usepackage{amsmath}
\usepackage{multirow}

\newcommand{\vX}{{\mathbf{X}}}
\usepackage{tablefootnote}
\usepackage{booktabs}
\usepackage{color, colortbl}
\definecolor{Gray}{gray}{0.9}
\usepackage{newfloat}
\usepackage{listings}
\usepackage{hyperref}
\hypersetup{
    colorlinks=true,
    linkcolor=red,
    urlcolor=blue,
    }

\DeclareCaptionStyle{ruled}{labelfont=normalfont,labelsep=colon,strut=off} 
\floatstyle{ruled}
\newfloat{listing}{tb}{lst}{}
\floatname{listing}{Listing}
\title{Effective Causal Discovery under Identifiable Heteroscedastic Noise Model}
\author{
    Naiyu Yin\textsuperscript{\rm 1},
    Tian Gao\textsuperscript{\rm 2},
    Yue Yu\textsuperscript{\rm 3},
    Qiang Ji\textsuperscript{\rm 1}
}
\affiliations{
    \textsuperscript{\rm 1}Rensselaer Polytechnic Institute, Troy, NY.\\
    \textsuperscript{\rm 2}IBM Research, Yorktown Heights, NY.\\
    \textsuperscript{\rm 3}Lehigh University, Bethlehem, PA.\\

}

\usepackage{bibentry}

\begin{document}

\maketitle

\begin{abstract}
Capturing the underlying structural causal relations represented by Directed Acyclic Graphs (DAGs) has been a fundamental task in various AI disciplines. Causal DAG learning via the continuous optimization framework has recently achieved promising performance in terms of both accuracy and efficiency. However, most methods make strong assumptions of homoscedastic noise, i.e., exogenous noises have equal variances across variables, observations, or even both. The noises in real data usually violate both assumptions due to the biases introduced by different data collection processes. To address the issue of heteroscedastic noise, we introduce relaxed and implementable sufficient conditions, proving the identifiability of a general class of SEM subject to these conditions. Based on the identifiable general SEM, we propose a novel formulation for DAG learning that accounts for the variation in noise variance across variables and observations. We then propose an effective two-phase iterative DAG learning algorithm to address the increasing optimization difficulties and to learn a causal DAG from data with heteroscedastic variable noise under varying variance. We show significant empirical gains of the proposed approaches over state-of-the-art methods on both synthetic data and real data. Our implementation: \url{https://github.com/naiyuyin/ICDH}.
\end{abstract}

\section*{Introduction}
Learning the statistical and causal dependencies of a distribution in the form of a directed acyclic graph (DAG) is of great interest in areas such as causal inference and Bayesian network structure learning. The underlying statistical or causal relations indicated by the DAG have been applied to various machine learning applications \citep{ott2004finding, spirtes1995causal}. Causal DAG plays an increasingly important role in many machine learning tasks, including out-of-distribution generalization~\citep{janzing2018detecting, shen2018causally, ahuja2021invariance}, domain adaptation~\citep{javidian2021scalable,stojanov2021domain}, and transfer learning~\citep{scholkopf2019causality}. 

The gold standard approach to performing causal discovery is to conduct controlled experiments, which can be expensive, time-consuming, and sometimes even infeasible. Therefore, algorithms have been proposed to learn a DAG from purely observational data. These algorithms can be divided into two categories: constraint-based methods and score-based methods. The constraint-based methods estimate  DAGs by performing independence tests between variables. Popular algorithms include PC \citep{spirtes2000causation} and FCI \citep{spirtes1995causal, zhang2008completeness}. The score-based methods search through the DAG space for a DAG with the optimal score. The differences among score-based methods usually come from search procedures, such as hill-climbing and Greedy Equivalent Search (GES) \citep{chickering2002optimal}. The structural causal model-based methods encode the statistical and causal dependencies via structural equation models (SEM). \citet{zheng2018dags} introduces a continuous DAG constraint and NOTEARS algorithm, which reformulates the original combinatorial DAG learning problem as a constrained continuous optimization. Such conversion enables the employment of continuous optimization techniques in follow-up works~\citep {kalainathan2018sam,yu2019dag,ng2019masked}.

Under either a linear or non-linear structural equation model (SEM) assumption, most of the current methods~\citep{zheng2018dags, zheng2020learning, yu2020dags, peters2014causal} usually adopt an assumption in SEM that the noises are additive to causal functions and are assumed to have equal variance for each variable. However, such an assumption may not hold in real-world data. For example, real-world data may be gathered from diverse sources, spanning different times and locations, employing a variety of collection techniques. As a result, the exogenous factors that impact each variable may differ, and noise variances become non-constant for observations. Incorrect assumptions regarding variable noise homoscedasticity, when they are heteroscedastic, may lead to inaccurate and biased estimates. Several works \citep{ng2020role, lachapelle2019gradient, park2020identifiability} seek to allow the noises of each variable to have different variances but fall short of fully addressing noise heteroscedasticity.

A few recent works explicitly extend the SEM with additive noise assumption to more general cases and estimate the noise observation heteroscedasticity. \citet{rajendran2021structure} employs SEM with multiplicative noise, while \citet{blobaum2018cause} assumes the existence of a joint distribution between noise and parent variables. \citet{lachapelle2019gradient}, \citet{xu2022causal, immer2022identifiability}, \citet{khemakhem2021causal}, \citet{duong2023heteroscedastic} modulate the noise variances as a deterministic function of the parent variables. However, these works adopt bivariate SEMs and infer pairwise causal relations. To learn a causal DAG with more than two variables, they need to estimate the causal order or the skeleton first using existing methods. 

To accurately estimate the DAG from data with heteroscedastic variable noises and varying residual variance across observations, we propose employing a more general form of SEM and, thus, designing a novel DAG structure learning formulation. The main advantage of using a general SEM lies in relaxing the assumptions on noise variances, allowing not only unequal variances across variables but also varying variances across observations for the same variables. Such relaxation reduces model misspecification and enables the algorithm to more accurately capture noise variances and learn DAGs from challenging yet realistic data. However, this relaxation also significantly increases the difficulties in optimization modeling~\citep{lachapelle2019gradient}.

\noindent\textbf{Main Contributions:} 
To tackle those issues, we make three major contributions: 1) We introduce relaxed, implementable sufficient conditions for the identifiability of a general class of multivariate SEM. Guided by the identifiability conditions, we propose a novel DAG learning formulation that considers the variability of noise variances both among variables and across observations. To achieve this, our formulation models the parameters of the noise distribution with neural networks (Eq.~\eqref{eq_conditiona_distribution}). 2) We present an effective and practical two-phase DAG learning algorithm, which iteratively minimizes the objective to ensure accurate estimation of noise variances and DAG. 3) Empirical results demonstrate that our method achieves comparable accuracy on synthetic homoscedastic noise data compared to state-of-the-art methods. Moreover, it significantly outperforms these methods on synthetic heteroscedastic data and real data.

\section*{Related works}
In an SEM with Gaussian additive noise, functional causal model-based methods, such as \citet{chen2019causal}, assume the variables have homoscedastic noises\footnote{If a set of variable noises is homoscedastic, then they have equal variances.} with equal noise variances across observations. In other words, the variable noises have equal variance across both variables and observations. The strong homoscedastic assumption is also implicitly posed for methods~\cite{zheng2018dags, zheng2020learning, yu2019dag, gao2020polynomial} that adopt reconstruction loss under the same SEM setting\footnote{Please refer to Appendix \ref{derivations_generality} for details. }. \citet{ng2020role} relaxes the homoscedastic variable noise assumption, allowing the noises of different variables to have non-equal variances. Similarly, \citet{lachapelle2019gradient} and \citet{park2020identifiability} perform the same relaxation.

Moreover, the above methods assume equal noise variances for each variable across observations, whereby the variable noise variance may vary from observation to observation due to the variation of the data collection conditions. Noise observation heteroscedasticity modeling has received increasing attention over the past few years. The general approach is to relax the independence between the parent variables and the additive noise. \citet{blobaum2018cause} allows a dependency between parent variables and the noise by assuming a joint distribution of two terms exists. \citet{xu2022causal} models the noise variance as a piece-wise function of the parent variables with limited choices of variance values. \citet{khemakhem2021causal, immer2022identifiability, duong2023heteroscedastic} employ a general form of SEM and modulate the noise variance as a deterministic function of the parent variables. However, \citet{khemakhem2021causal} and \citet{immer2022identifiability} are mainly designed to identify pair-wise cause-effect relations for bivariate SEMs. \citet{duong2023heteroscedastic} proposes to estimate the causal order and then orient the pair-wise causal directions for multivariate SEMs. An extension of GraN-DAG, denoted as GraN-DAG++, also estimates the noise variances as a function of parent variables and learns a DAG for the multivariate case. However, due to the heteroscedasticity complexity and optimization limitation, GraN-DAG++ learns at best comparable accurate DAG. \citet{rajendran2021structure} employs the multiplicative SEMs to model the heteroscedastic noise data but learn the causal structure via a discrete optimization framework.
We summarize the above methods in Table \ref{tab:methods}.
\begin{table*}[hpt]
    \centering
    \normalsize
    \begin{tabular}{l||c|c|c|c|c}
    \toprule
        \multirow{2}{*}{\textbf{SoTA Methods}} & \multicolumn{4}{c|}{\textbf{SEM}} & \textbf{Algorithm}  \\
          & $\#$ var. & Causal function & Noise & Identifiable & Optimization \\
         \hline
         NOTEARS~\citep{zheng2018dags} & Multivariate & Linear & Homo  & \checkmark & Continuous\\
         NOTEARS-MLP~\citep{zheng2020learning} & Multivariate & Nonlinear & Homo  & \checkmark & Continuous \\
         GOLEM~\citep{ng2020role} & Multivariate & Linear & Homo  & \xmark & Continuous \\
         GraN-DAG~\citep{lachapelle2019gradient} & Multivariate & Nonlinear & Homo  & \checkmark & Continuous\\
         GraN-DAG++~\citep{lachapelle2019gradient} & Multivariate & Nonlinear & Hetero  & \xmark & Continuous \\
         US\citep{park2020identifiability} & Multivariate & Linear & Hetero  & \checkmark & Combinatorial \\
         HEC~\citep{xu2022causal} & Bivariate & Nonlinear & Hetero  & \checkmark & Combinatorial \\
         CAFEL~\citep{khemakhem2021causal} & Bivariate & Nonlinear & Hetero  & \checkmark & Combinatorial\\
         LOCI~\citep{immer2022identifiability} & Bivariate & Nonlinear & Hetero  & \checkmark & Combinatorial\\
         GFBS~\citep{gao2020polynomial} &  Multivariate & Both & Hetero  & - 
         & Combinatorial\\ 
         
         HOST\citep{duong2023heteroscedastic} & Multivariate & Nonlinear & Hetero & \checkmark & Combinatorial \\
         \hline
        ICDH(Ours) & Multivariate & Nonlinear & Hetero  & \checkmark  & Continuous\\
        \bottomrule
        \end{tabular}
    \caption{\normalsize Summary of SEMs and algorithms for SoTA methods. "Homo" represents homoscedastic noise, and "Hetero" represents heteroscedastic noise. GFBS employs multiple linear and nonlinear SEMs, each with varying identifiability.}
    \label{tab:methods}
\end{table*}

In the following section, we first introduce the general form of SEM. Then we introduce sufficient conditions that provide theoretical justification for its identifiability on multivariate variables. We then propose a general DAG learning formulation, which cannot only accurately model the variation of noise variance across both variables and observations but also capture a more accurate DAG structure in complex and noisy real-world datasets or applications.

\section*{Background and formulation}
\subsection*{Preliminaries}\label{sec_Preliminaries}
\noindent\textbf{Structural Equation Model (SEM) with additive noise: } Let $X$ be a set of $N$ random variables, $X = [X_1, X_2, \cdots, X_N]$. The causal relations between a variable $X_n\in X$ and its parents can be modeled via Eq.~\eqref{eq_sem_g}:
\begin{equation}\label{eq_sem_g}
   X_n = f_n(X_{\pi_n}) + E_n, n = 1,2,\cdots,N
\end{equation}
where $f_n(\cdot)$ is the structural causal function. $X_{\pi_n}$ are the parent variables of $X_n$. $E_n$ is the exogenous noise variable corresponding to variable $X_n$. Together they account for the effects from all the unobserved latent variables and are assumed to be mutually independent~\citep{peters2011causal}.

\noindent\textbf{DAG structure learning under SEM: } To learn a DAG $\mathcal{G}$ from a given joint distribution $P(X)$, $X$ is usually modeled via SEMs defined by a set of continuous parameters {$A = (A_1, A_2, \cdots, A_N)$} that encode all the causal relations, i.e.,
\begin{equation}\label{eq_sem_A}
    X_n = f_n(X; A_n) + E_n, n = 1, 2, \cdots, N
\end{equation}
{where $A_n$ are the parameters in each SEM.}
Compared to Eq.~\eqref{eq_sem_g}, it is easy to see that $A_n$ selects parent variables $X_{\pi_n}$ for each $X_n$. The goal is to estimate $A$, based on which we can infer the DAG $\mathcal{G}$. Let $\vX\in\mathbb{R}^{M\times N}$ denote the input matrix of $M$ observations of the random variable set $X$. Given $\vX$, $A$ is estimated by minimizing the loss function $F(\vX, A)$, subject to the continuous acyclicity constraint $h(A) = tr(e^{A \circ A}) - N= 0$~\citep{zheng2018dags, zheng2020learning}\footnote{The continuous DAG constraints for linear SEM and nonlinear SEM are introduced respectively in \citep{zheng2018dags} and \citep{zheng2020learning}. We use $h(Z)$ to refer that the acyclicity constraint is posed on parameters $Z$, regardless of SEM types.}
\begin{equation}\label{eq_general_optimization}
  \begin{split}
        A^* &= \arg\min_A  F(\vX, A) \qquad \text{subject to } h(A) = 0
    \end{split}
\end{equation}
where $F(\vX, A)$ evaluates the negative log-likelihood of $A$ as the underlying relations encoded in $\vX$. The parameterization with $A$, along with the introduction of continuous acyclicity constraint, transforms the DAG learning under SEM into a continuous optimization problem and enables the usage of powerful optimization techniques.

\noindent\textbf{General SEM and identifiability issue.} To ensure the employed SEMs are identifiable, i.e., a unique graph $\mathcal{G}$ can be identified from the joint distribution $P(X)$ generated from SEMs, the exogenous variable is usually assumed to be additive (Eq.~\eqref{eq_sem_g}). The general SEMs in Eq.~\eqref{eq_sem_general} are proven to be unidentifiable without any constraint~\citep{zhang2015distinguishing}.
\begin{equation}\label{eq_sem_general}
    X_n = f_n(X_{\pi_n}, E_n), n = 1, 2, \cdots, N
\end{equation}
However, some recent works try to investigate the identifiability of the general SEM with weaker assumptions and develop DAG learning methods based on the identifiable SEM. 
\subsection*{Problem statement}\label{sec_Proposed_formulation}
We first introduce one of the general SEM formulations in \textbf{Definition \ref{def:HNM}} that modulate the noise variance with cause variables. The SEMs we consider in \textbf{Definition \ref{def:HNM}} are about SEMs with heteroscedastic additive noise. It generalizes the causal function $f_n(\cdot)$ in Eq.~\eqref{eq_sem_A} from merely an additive transformation of causes and exogenous noise to both affine and additive transformation.
Such generalization increases the model's ability to approximate data with more complex types of noise. The general SEM can address data heteroscedasicity, whereby the noise variances vary across variables and observations, depending on the causes. 
\begin{definition}\label{def:HNM}
(Heteroscedastic noise model) The SEMs are heteroscedastic noise models (HNMs) if Eq.~\eqref{eq_hnm} holds for each $X_n\in X$, 
\begin{equation}\label{eq_hnm}
    X_n = f_n(X_{\pi_n}) + {\sigma_n(X_{\pi_n})}E_n, {n=1,2,\cdots, N}
\end{equation}
where $E_1, E_2, \cdots, E_n$
are statistically independent and all follow Gaussian distributions. $\sigma_n(X_{\pi_n})>0$.
\end{definition}
The investigation of DAG learning methods under HNM has been increasingly studied due to its flexibility in modeling more complex and general data generation processes in realistic data. Let $\mathbb{E}[E_n|X_{\pi_n}] = 0$ and $Var[E_n|X_{\pi_n}] = 1$, then the conditional distribution under HNM $p(X_n|X_{\pi_n})\sim \mathcal{N}\big(f_n(X_{\pi_n}), \sigma^2_n(X_{\pi_n})\big)$.  

\noindent\textbf{Advantages of HNM:} We choose the SEM that modulates noise variances with cause variables for three reasons. First, it relaxes the strong independence assumption between exogenous variables and observed variables. Secondly, it satisfies the assumed data generation process, whereby observations for each variable are generated using their cause variables. Moreover, it is easy to implement via deep neural networks, which are known for their ability to
modeling complex data distributions.

\noindent\textbf{Limitations of prior works under HNM:} \citet{xu2022causal} models the variance $\sigma_n$ as a deterministic piece-wise 
function of the parent variables, which limits the approximation of the variances to a few choices. 
\citet{khemakhem2021causal} limits their choices of $f$ to be nonlinear and invertible functions to ensure identifiability. However, this identifiable condition cannot readily be extended to multivariate cases. For the bivariate case, the invertibility of $f$ is easily satisfied since its inputs and outputs are values of a single variable. For the multivariate cases, the input into the $f_n$ is the parent variables $X_{\pi_n}$ of variable $X_n$. The dimensions match only when the number of parent variables is 1. There is no guarantee that an invertible function $f_n$ exists for $X_n$. 
\citet{duong2023heteroscedastic} proposes to learn the causal DAG by first searching for the causal order and orienting edges subject to the obtained order. However, its performance is susceptible to the accuracy of independence tests,
which can be challenging to perform with difficult data. Early errors in order estimation can propagate to later stages of causal direction orientation, causing the algorithm to learn inaccurate causal graphs. Moreover, due to the time complexity of subset independence tests, the algorithm cannot scale up to large models.
 
Therefore, \textbf{our goal is to formulate the DAG learning problem under the identifiable multivariate HNM into a continuous optimization framework and solve the optimization with powerful tools such as neural networks.} To do so, we first introduce relaxed implementable sufficient conditions that provide identifiability for multivariate HNM in section \ref{sec: ident-cond}. Guided by those conditions, we propose our continuous DAG learning formulation in section \ref{sec: formulation-HNM}. 

\subsection*{Proposed identifiable HNM}\label{sec: ident-cond}
In this section, we introduce the sufficient conditions for the HNM to uniquely identify a DAG from the given data distribution in \textbf{Theorem \ref{theorem: identifiable}}.
We can theoretically prove that the HNM is identifiable if those sufficient conditions hold. 

\begin{theorem}\label{theorem: identifiable} (Identifiability)
The formulation in Eq.~\eqref{eq_hnm} is identifiable if the following conditions are satisfied: 1) $f_1, f_2, \cdots, f_N$ are nonlinear; 2) $\sigma_1, \sigma_2, \cdots, \sigma_N$ are piece-wise  
functions. 3) $E_1, E_2, \cdots, E_N$ are independent and follow Gaussian distributions\footnote{{$E_n$s are i.i.d Gaussian is a sufficient but not necessary condition of identifiability. By assuming i.i.d. Gaussian noise, sufficient conditions allow the HNM for one direction to exist under the bivariate case, and serve as the most essential lemma for our identifiability theorem.}}.
\end{theorem}
Please refer to the Appendix \ref{proof_identifiability} for all proofs.

The nonlinearity for $f_n$ is in terms of $\forall X_j \in X_{\pi_n}$. The nonlinearity in terms of each input variable is slightly stronger than the nonlinearity in terms of the input parent set. However, it is easy to satisfy if we employ deep neural networks as $f_n$s because the nonlinear activation function is applied to each dimension of the inputs.

\noindent\textbf{Comparison with identifiable PNL:} The identifiable post-nonlinear model (PNL) in \citet{zhang2012identifiability} assumes the SEM between a variable $Y$ and its cause $X$ follows $Y=f_2(f_1(X) + N)$, where $N$ is the independent noise. They further assume $f_2$ to be a fixed non-invertible function. Compare the PNL to the HNM, there exist cases that can be proved identifiable and covered by one model but not the other. Hence, it is impossible to compare the flexibility of the two models. They are developed to address the identifiability of different classes of SEMs.
\subsection*{Proposed formulation}\label{sec: formulation-HNM}
To perform DAG learning under identifiable multivariate HNM, we parameterize Eq.~\eqref{eq_hnm} with a set of continuous parameters that enforce the formulation to satisfy the identifiability conditions. We instantiate Eq.~\eqref{eq_hnm} with continuous parameters $A$ and $B$, where $A, B$ are the parameters for causal functions $f = (f_1, f_2, \cdots, f_N)$ and variances estimation functions {$\sigma = (\sigma_1, \sigma_2, \cdots, \sigma_N)$}. 
Hence, the Eq.~\eqref{eq_hnm} can be then re-written as:
\begin{equation}\label{eq_hnm_AB}
    \scalebox{0.95}{$
    \begin{split}
        X_n &= f_n(X, A_n) + \sigma_n(X, B_n)E_n, n=1,2,\cdots, N \\
    \end{split}
    $}
\end{equation}
There are three identifiability conditions to satisfy according to \textbf{Theorem \ref{theorem: identifiable}}. To satisfy condition (3), we assume $E_n\sim \mathcal{N}(0, 1)$ for $n=1,2,\cdots, N$. Then we adopt 2-layer Multi-layer Perceptrons (MLPs) for $f_n(\cdot)$s and $\sigma_n(\cdot)$s. By setting the activation functions as sigmoid functions for $f_n$s , ReLU functions for $\sigma_n$s, conditions (1) and (2) are satisfied. We use a 2-layer MLP in our formulation for simplicity. The number of layers and hidden neurons can vary as long as conditions (1) and (2) hold.

Besides the three conditions to ensure the identifiability, an underlying assumption in Eq.~\eqref{eq_hnm_AB} is that the parent variables that are input into functions $f_n$ and $\sigma_n$ should be the same, or are selected from the same set.
To ensure that such an assumption is always satisfied in our formulation, we design $A$ and $B$ to share partial parameters. In particular, we let the MLPs for $f_n$ and $\sigma_n$ share the first layer weights. We denote the first layer weights of $f_n$ as $W^{(1)}_n$, the second layer weights as $W^{(2)}_n$, hence we have
\begin{equation}\label{eq_f_nn}
\scalebox{0.95}{$
f_n(X, A_n) = {f_n(X, W_n^{(1)}, W_n^{(2)})}= W_n^{(2)}s(W_n^{(1)}X^T)$}
\end{equation}
where $W_n^{(1)}\in\mathbb{R}^{m_1\times N}, W_n^{(2)}\in\mathbb{R}^{1\times m_1}$. $A_n = (W_n^{(1)}, W_n^{(2)})$. $s(\cdot)$ is the sigmoid activation function. We let $\sigma_n$ share the first layer weights as $f_n$ and denote the second layer weights for $\sigma_n$ as $W^{(3)}_n$. We use a scalar parameter $W^{(3)}_{n0}$ to ensure the strict positivity of $\sigma_n$. Hence we have
\begin{equation}\label{eq_g_nn}
    \begin{split}
        \sigma_n(X, B_n) =& \sigma_n(X, W_n^{(1)},W_n^{(3)}, W_{n0}^{(3)}) \\
        =& \text{ReLU}\big(W_n^{(3)}s(W_n^{(1)}X^T)\big) + e^{W_{n0}^{(3)}}
    \end{split}
\end{equation}
where $W_n^{(3)}\in\mathbb{R}^{1\times m_1}$. $W_{n0}^{(3)} \in \mathbb{R}$. $B_n = (W_n^{(1)}, W_n^{(3)}, W_{n0}^{(3)})$. We place the acyclicity constraint on the shared parameters $W^{(1)} = (W^{(1)}_1, W^{(1)}_2, \cdots, W^{(1)}_N)$ to enforce the $W^{(1)}$ to encode causal relations. Intuitively, we assume there is one unique $\mathcal{G}$, represented by the weighted matrices $W^{(1)}$. $W^{(2)} = (W^{(2)}_1, W^{(2)}_2, \cdots, W^{(2)}_N), W^{(3)} = (W^{(3)}_1, W^{(3)}_{10}, W^{(3)}_2, W^{(3)}_{20}, \cdots, W^{(3)}_N, W^{(3)}_{N0})$ are the parameters to estimate the the mean and variance using parent sets selected by $W^{(1)}$. $W^{(2)}, W^{(3)}$ may further select subsets from the parent sets for estimation. We infer our estimation of the DAG $\mathcal{G}$ from $W^{(1)}$.

\noindent\textbf{Advantages of sharing parameters:} The formulation that shares $W^{(1)}$ automatically ensures that $f_n$s and $\sigma_n$s employ the same set of parent variables as inputs. Without parameter sharing, we need to impose additional constraint that enforces the DAG structures we inferred from $f_n$s and $\sigma_n$s separately to be consistent with each other. Moreover, the algorithm without parameter sharing may also suffer from increased time complexity, due to the enforcement of time-consuming acyclicity constraints on parameters from both $f_n$s and $\sigma_n$s.

\subsection*{Optimization objective and difficulties}
The goal is to estimate a DAG $\mathcal{G}$, given $M$ observations of $X$, i.e., input matrix $\vX = \{\vX(m)\}_{m=1}^M$. $\vX(m)\in \mathbb{R}^{1\times N}$ is the $m^{th}$ observation of $X$. {$\vX(m) = [\vX_1(m), \vX_2(m), \cdots, \vX_N(m)]$, where $\vX_n(m)$} is the $m^{th}$ observation of variable $X_n$. 
{According to the HNM, the variance for $\vX_n(m)$ can be modeled via $\sigma^2_n(\vX(m), B_n)$.} Since $E_n \sim \mathcal{N}(0, 1)$, given $\vX(m)$, the conditional distribution of the $m^{th}$ observation corresponding to variable $X_n$ given its parent variables $\vX_{\pi_n}(m)$, {i.e. $\vX_n(m)$}, can be modeled as:
\begin{equation}\label{eq_conditiona_distribution}
    \scalebox{0.80}{$
    p(\vX_n(m)|\vX_{\pi_n}(m)) \sim \mathcal{N}\Big(f_n\big(\vX(m), A_n\big), \sigma^2_n(\vX(m), B_n)\Big)$}
\end{equation}
Based on Eq.~(\ref{eq_conditiona_distribution}), we derive the negative log-likelihood of the marginal distribution $p(\vX)$ as the objective in our proposed formulation:
\begin{equation}\label{eq_nll}
    \scalebox{0.8}{$
    \begin{split}
    \mathcal{L}_{nll}(\vX, A, B) =& 
    \sum_{m,n=1}^{M,N} \Big[\log\big(\sigma_n(\vX(m), B_n)\sqrt{2\pi} \big) \\
    &+ \frac{\big(\vX_n(m) - f_n(\vX(m), A_n)\big)^2}{2\sigma^2_n(\vX(m), B_n)}\Big] 
    \end{split}$}
\end{equation}
The detailed derivation can be found in Appendix \ref{nll_derivation}. 

Substituting the Eq.~\eqref{eq_f_nn} and \eqref{eq_g_nn} into negative log-likelihood loss in Eq.~(\ref{eq_nll}), we obtain the training objective under proposed formulation w.r.t $W^{(1)}, W^{(2)}$, and $W^{(3)}$:
\begin{equation}\label{nll_loss_W}
\scalebox{0.75}{$
\begin{split}
     & \mathcal{L}_{nll}(\vX, W^{(1)}, W^{(2)}, W^{(3)}) \\
     =& \sum_{m, n=1}^{M, N}\Big[\log \sqrt{2\pi} + \log [\text{ReLU}\Big(W_n^{(3)}s\big(W_n^{(1)}\vX^T(m)\big)\Big)  + e^{W_{n0}^{(3)}}] \\
     &+ \frac{\Big(\vX_n(m) - W_n^{(2)}s\big(W_n^{(1)}\vX^T(m)\big)\Big)^2}{2\big[\text{ReLU}\Big(W_n^{(3)}s\big(W_n^{(1)}\vX^T(m)\big)\Big)  + e^{W_{n0}^{(3)}}\big]^2 }\Big]
\end{split}$}
\end{equation}
The DAG learning problem becomes the constrained continuous optimization that finds the optimal values $(W^{(1)})^*, (W^{(2)})^*, (W^{(3)})^*$ by minimizing $\mathcal{L}_{nll}(\vX, W^{(1)}, W^{(2)}, W^{(3)})$ subject $h(W^{(1)}) = 0$.

Intuitively, by introducing and estimating conditional distribution variances $\bm \sigma = \{\sigma^2_n(\vX(m), B_n)\}_{n,m=1}^{N, M}$ as functions of causes in HNM, our formulation allows the modeling of heteroscedasticity within the data noise. However, on the other hand, $\bm \sigma$ estimation inevitably increases modeling and optimization difficulties significantly, causing state-of-art global DAG learning methods like GraN-DAG++~\cite {lachapelle2019gradient} to fail. 

The difficulty of learning the causal DAG under the proposed formulation lies in effectively minimizing the negative log-likelihood loss over two sets of parameters $A$ and $B$ jointly while the interplay between optimization over $A$ and $B$ compromises the accuracy of each other. If the algorithm jointly learns $A, B$, the optimization process tends to minimize the negative log-likelihood loss by learning a set of $B$ that significantly increases the estimated $\bm \sigma$. As a result, the algorithm can reach a stationary solution without enforcing the residual errors to be small. To solve such difficulties, we propose a DAG learning approach based on a two-phase algorithm, which estimates causal functions parameters $A$ and $\bm \sigma$ estimation parameters $B$ alternatively and iteratively. 
\begin{figure*}
    \centering
    \normalsize
    \includegraphics[width=1\linewidth]{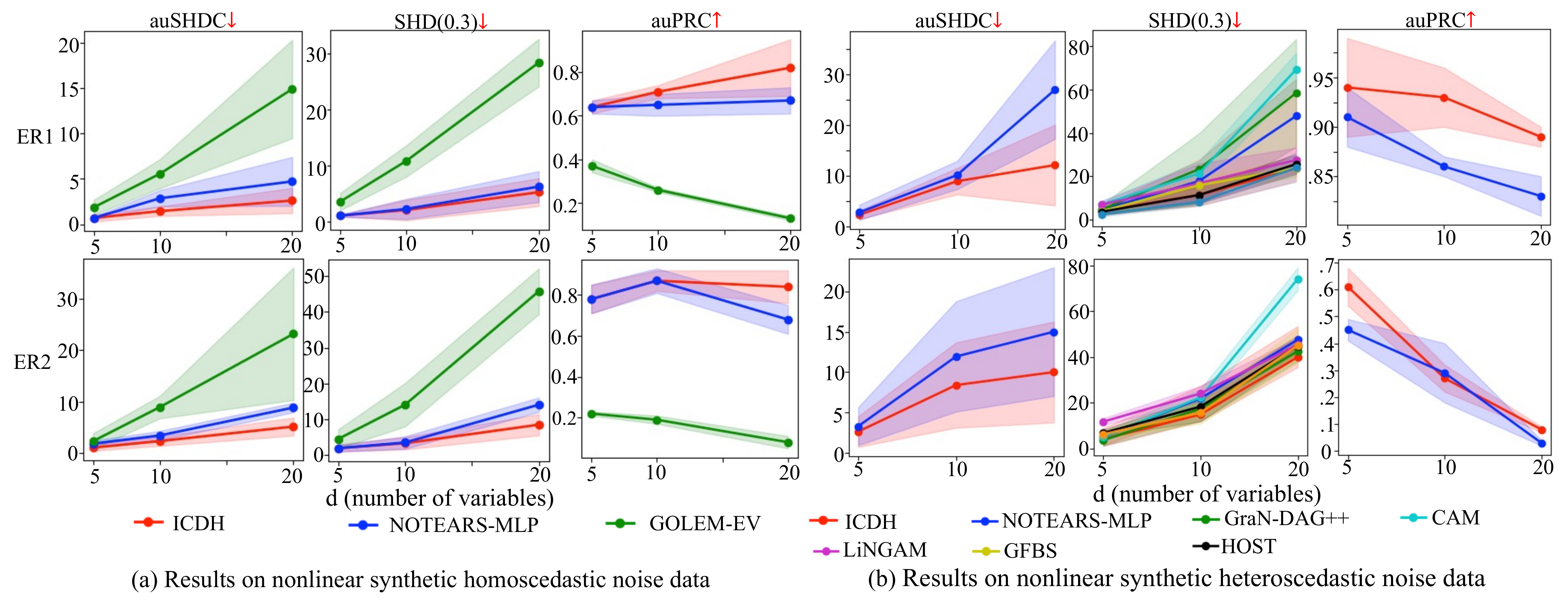}
    \caption{Comparison of SoTA baselines on synthetic data: results(mean, standard error) on auSHDC, SHD, and auPRC. We only report auSHDC and auPRC on baselines directly return weighted adjacent matrices. Our method is shown in the red curve.}
    \label{fig:syn_res}
\end{figure*}

\section*{Two-phase iterative learning algorithm}
As we mentioned above, we introduce and model the parameters for conditional distribution variances $\bm \sigma$ into our model. To avoid the interplay between optimization over mean and variance parameters of conditional distributions, we propose to first estimate the variances $\bm \sigma$ and then estimate mean parameters under fixed variance. To provide mathematical justification for such an iterative learning approach, we introduce posterior distribution for variance $q$ in Eq.~\eqref{EM-deri}. For simplicity, we denote $\sigma_n^2(\vX(m), B_n)$ in Eq.~\eqref{eq_conditiona_distribution} as $\sigma^2_n(m)$, and $\bm \sigma^2(m) := \{\sigma^2_n(m)\}_{n=1}^N$, $\bm \sigma^2  := \{\bm \sigma^2(m)\}_{m=1}^M$. 
Hence we can write the marginal log-likelihood of $\vX$ as follows:
\begin{equation}\label{EM-deri}
    \scalebox{0.85}{$
    \begin{split}
        \log p(\vX | A) 
       &\geq \int_{\bm \sigma^2} q(\bm\sigma^2|\vX, \Theta_q)\log \frac{p(\vX, \bm \sigma^2|A)}{q(\bm \sigma^2|\vX, \Theta_q)}\diff \bm \sigma^2 
    \end{split}$}
\end{equation}
We drop the entropy term \scalebox{0.85}{$q(\bm \sigma^2|\vX,\Theta_q)\log q(\bm \sigma^2|\vX,\Theta_q)$}, since we consider the $\Theta_q$ is independent of current parameters $A$. The objective is to maximize the lower bound of the marginal log-likelihood:
\begin{equation}
\scalebox{0.85}{$
    \begin{split}
        A^* = \arg\max_{A}\int_{\bm \sigma^2} q(\bm \sigma^2|\vX, \Theta_q)\log p(\vX, \bm \sigma^2|A)\diff \bm \sigma^2
    \end{split}$}
\end{equation}
{We use $t$ as the notation for the iteration index of our proposed algorithm.} We chose {$\Theta_q^t$ to be $A^{t-1}$, i.e., set $\Theta_q$ in the current iteration with $A$ from the previous iteration,} $q(\bm \sigma^2|\vX, \Theta_q^t) = p(\bm \sigma^2|\vX, A^{t-1})$. This selection of $q$ has been proven to be a tight lower bound of $p$. To simplify the learning procedure, we obtain the optimal value of the $\bm \sigma^2$, denoted as $\hat{\bm \sigma}^2$, via maximizing the $p(\bm \sigma^2|\bm X, A^{t-1})$. Phase-I and Phase-II can be performed as follows.
\begin{equation}\label{eq_Estep}
    \text{Phase-I}: \qquad   \hat{\bm \sigma}^2 = \arg\max_{\bm \sigma^2} p(\bm \sigma^2|\vX, A^{t-1})
\end{equation}
\begin{equation}\label{eq_Mstep}
    \text{Phase-II}: \qquad    A^* = \arg\max_{A} \log p(\vX, \hat{\bm \sigma}^2|A)
\end{equation}
In Phase-I, to obtain the posterior distribution $p(\bm \sigma^2|\vX, A^{t-1})$, we assume there exists a non-informative uniform prior $p(\bm\sigma^2)$\footnote{{We choose a non-information prior for $p(\bm \sigma^2)$, which is the least restrictive so that we can simplify the objective into mere likelihood. Our formulation can also adapt to other types of the prior distribution.} }. Then the posterior distribution is proportional to the likelihood of the marginal distribution $p(\vX|\bm \sigma^2, A^{t-1})$, i.e., $p(\bm \sigma^2|\vX, A^{t-1}) \propto p(\vX|\bm \sigma^2, A^{t-1})$. The optimal estimation of variances can be obtained by maximizing the likelihood of the marginal distribution $p(\vX|\bm \sigma^2, A^{t-1})$, or minimizing its log-likelihood, i.e., the NLL loss in Eq.~(\ref{eq_nll}) with $A = A^{t-1}$. Given $\vX$, the values of $\bm \sigma^2$ depend on the parameters in $B$ that are not shared with $A$. The optimization in Eq.~(\ref{eq_Estep}) can be simplified to set $W^{(1)} = (W^{(1)})^{t-1}, W^{(2)} = (W^{(2)})^{t-1}$  and optimize $W^{(3)}$ over $\mathcal{L}_{nll}$:
\begin{equation}\label{eq_Estep_final}
\scalebox{0.85}{$
\begin{split}
    (W^{(3)})^* &= \arg\min_{W^{(3)}} \mathcal{L}_{nll}(\vX, (W^{(1)})^{t-1}, (W^{(2)})^{t-1}, W^{(3)}) \\
    \hat{\sigma}_n^2(m) &= \sigma_n(\vX(m), (W^{(1)})^{t-1}, (W^{(3)})^*)\\
    n &= 1, 2, \cdots, N, m = 1, 2, \cdots, M
\end{split}
    $}
\end{equation}
In Phase II, we directly maximize the likelihood given the optimal
estimation of variances, or in practice minimize the NLL loss in Eq.~(\ref{eq_nll}) given $\bm \sigma^2 = \hat{\bm \sigma}^2$. The optimization in Eq.~(\ref{eq_Mstep}) can be simplified to optimize $W^{(1)}, W^{(2)}$ in $A$ over $\mathcal{L}_{nll}$ with fixed values for variances and subject to the acyclicity constraint on $W^{(1)}$:
\begin{equation}
    \scalebox{0.85}{$
    \begin{split}
        (W^{(1)})^*, (W^{(2)})^* &= \arg\min \mathcal{L}_{nll}(\vX, W^{(1)}, W^{(2)}, \hat{\bm \sigma}^2) \\
        &\text{subject to }h\big(W^{(1)}\big) = 0
    \end{split}$}
\end{equation}
We choose to only update $W^{(3)}$ in Phase-I to prevent poor empirical performance caused by the joint optimization over two competing terms in our loss. If $W^{(1)}$ is jointly optimized in Phase-I, we will obtain a degenerate solution with large reconstruction loss and larger unreasonable variances. To solve the constrained continuous optimization problem in Phase II, we adopt a standard Lagrangian optimization process and force $W^{(1)}$ to satisfy the acyclicity constraint (Algorithm 3). The augmented Lagrangian optimization method is generally accepted as the better method, compared to the alternative penalty method~\cite{ng2022convergence}. We choose ALM for a fair comparison since it has been employed by many state-of-the-art methods that tackle the same issue as our method. 
We outline the full procedure (Algorithm \ref{algorithm1}), Phase I procedure (Algorithm \ref{algorithm2}), Phase II procedure (Algorithm \ref{algorithm3}) 
in Appendix \ref{pseudo-code}.

\noindent\textbf{Convergence analysis:} Our proposed two-phase iterative learning approach can only guarantee a stationary solution, i.e., the gradients of parameters $A$ and $B$ w.r.t our training objective can achieve zeros after the algorithm converges. Please refer to Appendix \ref{sec_EM_Convergence_gaurantee} for details.

\noindent\textbf{Complexity analysis:} In Phase II, the time complexity is $\mathcal{O}(N^3)$ w.r.t number of nodes $N$, which takes the same number of optimization iterations as other continuous methods with Augmented Lagrangian Method (ALM)~\citep{zheng2020learning, lachapelle2019gradient}. Phase I is relatively much cheaper in computation. The time complexity is $\mathcal{O}(mN^2)$ as one iteration of LBFGS with memory size $m$ is employed. The total time complexity of our algorithm is $\mathcal{O}(kN^3)$ with $k$ iterations of two phases ($k\leq 5$ in practice).
Our proposed method has the same order of magnitude as the other baseline methods, and can
handle the same amount of variables.

\section*{Experiment}
We perform experiments on real data and synthetic data to demonstrate the effectiveness of our proposed method. We denoted our method as \textbf{I}dentifiable \textbf{C}ausal \textbf{D}iscovery under \textbf{H}eteroscedastic data(ICDH). For more details on synthetic data generation procedure and evaluation metrics, please refer to the Appendix \ref{sec_data_description} and \ref{eval_metrics} respectively.

\noindent\textbf{Baselines. }We compare our method against DAG learning methods using continuous optimization that also relaxes the strong assumptions of SEM: GOLEM-NV-L1~\citep{ng2020role}, GOLEM-EV-L1~\citep{ng2020role}, GraN-DAG~\citep{lachapelle2019gradient}, GraN-DAG++~\citep{lachapelle2019gradient}; the methods also address the heteroscedastic noise issue but under combinatorial optimization framework: HEC~\citep{xu2022causal} and CAREFL~\citep{khemakhem2021causal}, GFBS~\citep{rajendran2021structure} and HOST~\citep{duong2023heteroscedastic}; popular baselines NOTEARS-MLP~\citep{zheng2020learning}, CAM~\citep{peters2014causal}, LiNGAM~\citep{shimizu2014lingam}, and GES~\citep{chickering2002optimal}. 
\citet{xu2022causal, khemakhem2021causal} aim to learn pairwise causal relations instead of global graph structures. Hence we can only show the comparison on the cause-effect pairs dataset. 
\begin{table}[hpt]
\centering
\small 
    \begin{tabular}{c||c|c|c}
    \toprule
    Metrics & auSHDC$\downarrow$ &  SHD$\downarrow$ & auPRC$\uparrow$  \\ 
    \hline
    NOTEARS-MLP & 21.95 & 15 & 0.3427 \\
    GOLEM-EV & 25.41 & 17 & 0.1697\\
    GOLEM-NV & 26.53 & 14 & 0.2524\\
    GraN-DAG & - & $\bf{13}$ 
    & -\\
    GraN-DAG++ & - & $\bf{13}$
    & - \\
    GFBS & - & $17$
    & - \\
    HOST & - & $\bf{13}$ & - \\
    \hline
    \rowcolor{Gray}
    ICDH (ours)& $\bf{19.27}$  & $\bf{13}$ & $\bf{0.4673}$\\
    \bottomrule
  \end{tabular}
\caption{\normalsize Comparison of SoTA methods on Sachs dataset. 
 }
\label{tb:real}
\vspace{-5 mm}
\end{table}

Our method is not designed for heterogeneous and scale-invariant data. Tasks and assumptions in methods for heterogeneous data~\citep{huang2020causal, zhou2022causal} and scale-invariant data~\citep{reisach2021beware} differ from ours, making comparisons unfair on synthetic data tailored to our problems. Heteroscedastic noise may lead these methods to misestimate marginal variance and identify the wrong causal order. For a thorough comparison, we experiment with CD-NOD and sortnregress on heteroscedastic noise data (Table \ref{tab:comprehensive} from Appendix \ref{comprehnsive_comparision}), demonstrating our method's superior performance. Our focus is on developing a general algorithm under heterogeneous noise models for static data. Thus, we refrain from comparing with methods for temporal causal relations or those using complex noise distributions without explicitly modeling noise variance variation, as they are not relevant to this paper. 

\subsection*{Empirical results on synthetic data}
We generate synthetic data with different types of additive noises: homoscedastic noise with equal noise variances across variables and heteroscedastic noise. We also generated and experimented on homoscedastic noise with unequal noise variances across variables.

For each type of synthetic data, we compared different baselines based on the matchness between model formulations and data assumptions. The empirical results on homoscedastic equal noise data and heteroscedastic data are shown in Figure \ref{fig:syn_res}. Compared to the other SCM-based methods under a continuous optimization framework, empirical results indicate that our method can achieve comparable accuracy on homoscedastic noise data while outperforming baselines on heteroscedastic noise data. 
Compared to other types of methods, our method outperforms CAM, LiNGAM, and GFBS. Compared to GES and HOST, our method achieves comparable accuracy on data generated by sparse graphs and better performance on data generated by dense graphs. We also applied our method on the larger dataset with $50$ variables. On ER1 graphs, our ICDH method achieves SHD of $134.5\pm23.4$, outperform NOTEARS-MLP($144.1\pm38.0$), GraN-DAG++ ($161.1\pm30.8$), and HOST($152.5\pm24.8$). The effectiveness of our method on dense graphs can be verified by empirical results on $ER3$ graphs in Table \ref{tab:dense_graph} from Appendix \ref{dense_graphs}.
\vspace{-3 mm}
\subsection*{Empirical results on real data}
The empirical results on synthetic data, no matter homoscedastic or heteroscedastic, only indicate that the algorithms tend to perform well on the data that satisfies their model assumptions. These model assumptions are usually violated in real data or applications. Hence, a general formulation and an empirically effective learning approach are essential to solve real-world problems. We apply our method and the baseline methods on the two widely-studied real datasets: Sachs and cause-effect pairs. 
\begin{table}[hpt]
\centering
\small 
\begin{tabular}{c||c|c}
    \toprule
    Methods & 
    Accuracy $\uparrow$
    & Weighted Accuracy $\uparrow$ \\ 
    \hline
    NOTEARS-MLP & 39/99 & 0.49 \\
    NOTEARS & 36/99 & 0.47\\
    GOLEM-EV & 33/99 & 0.40\\
    GOLEM-NV & 33/99 & 0.40\\
    \rowcolor{Gray}
    ICDH(ours) & 52/99 & 0.58\\
    \hline
    HEC & - & 0.71 \tablefootnote{Reported results from \cite{xu2022causal}}\\
    CAREFL  & - & 0.73 \tablefootnote{Reported results from \cite{khemakhem2021causal}} \\ 
    \bottomrule
  \end{tabular}
  \caption{\normalsize Comparison of SoTA methods on cause-effect pairs dataset: results on 
 accuracy (number of correct inferences of cause-effect relations) and the weighted accuracy.}
  \label{tb:cep}
  \vspace{-5 mm}
\end{table}

\noindent\textbf{Sachs Dataset. } 
The results are summarized in Table~\ref{tb:real}. Our SHD of 16 for NOTEARS-MLP closely aligns with and is lower than the SHD of 17 reported in their paper. GraN-DAG, GraN-DAG++, GFBS, and HOST use post-processing to find the optimal DAG with minimal SHD. We achieve SHDs of 13 for GraN-DAG, GraN-DAG++, and HOST, consistent with their original papers. For the GFBS method, we achieve an SHD of 17. Empirical results demonstrate that our proposed method attains comparable accuracy (SHD of $13$) with state-of-the-art methods.

\noindent\textbf{Cause-effect pairs dataset. } 
Following standard experimental procedures, we focus on the 99 remaining bivariate problems, as summarized in Table \ref{tb:cep}. Our method correctly infers 52 out of 99 cause-effect pairs, outperforming all the other whole DAG learning methods: NOTEARS-MLP, NOTEARS, GOLEM-EV, and GOLEM-NV, which correctly identify 39, 36, 33, and 33 pairs, respectively. Our method achieves a lower weighted accuracy compared to HEC and CAREFL. Despite similar model assumptions, these methods are tailored for bivariate data, directly comparing models $X \leftarrow Y$ and $X \rightarrow Y$ to select the one with a higher proposed objective value. Our whole DAG learning method, relying on continuous optimization, may not find the global optimal objective. Furthermore, empirical results in Tables \ref{tb:real}-\ref{tb:cep} suggest real data likely involves heteroscedastic variables with varying noise variances across samples. Our DAG learning method, with a general model formulation and effective learning approach, proves more suitable for real-world data applications.
\vspace{-3 mm}
\section*{Conclusion}
In this paper, we introduce relaxed implementable sufficient conditions to provide the identifiability for a general class of multivariate SEM. We propose a novel formulation for the DAG learning problem guided by the conditions, which accounts for the noise variance variation across both variables and observations. Our formulation is identifiable and can generalize existing formulations of state-of-art methods. We then propose an effective two-phase iterative DAG learning approach to address the increasing training difficulties introduced by the general formulation. 
Empirical results show that our method achieves comparable accuracy on homoscedastic noise data while outperforming the SOTA methods on heteroscedastic noise data and real data, which indicates 1) the existing methods likely suffer when noise variances vary across observations, 2) our method has great potential for real-world applications.

\noindent\textbf{Acknowledgement:} This work is supported in part by the Rensselaer-IBM AI Research Collaboration (http://airc.rpi.edu), part of the IBM AI Horizons Network,
and by the National Science Foundation award IIS 2236026.

\bibliography{aaai24}

\newpage
\appendix
\section{The Derivation of the NLL Loss under the Proposed Formulation}\label{nll_derivation}
In this section, we provide the detailed derivation of the NLL loss in 
Eq. (10) and Eq. (11) 
in the main paper. 
{
\begin{equation}\label{eq_nll_detailed}
\scalebox{0.85}{$
    \begin{split}
        & \mathcal{L}_{nll} (\bm X, A, B)\\
        &=-\log p(\vX) \\
        &= -\log \prod_{m=1}^M p(\vX(m))\\
        &= -\log \prod_{m=1}^M \prod_{n=1}^N p(\vX_n(m)| \vX_{\pi_n}(m)) \\
        &= -\log \prod_{m=1}^M \prod_{n=1}^N \frac{1}{\sigma_n \big(\vX(m), B_n\big)\sqrt{2\pi}}e^{-\frac{\big(\vX_n(m) - f_n(\vX(m), A_n)\big)^2}{2\sigma_n^2 \big(\vX(m), B_n\big)}} \\
        &= \sum_{m=1}^M\sum_{n=1}^N\Big[\log\Big(\sigma_n \big(\vX(m), B_n\big)\sqrt{2\pi}\Big) + \frac{\big(\vX_n(m) - f_n(\vX(m), A_n)\big)^2}{2\sigma_n^2 \big(\vX(m), B_n\big)}\Big] \\
        &=\sum_{m=1}^M\sum_{n=1}^N\Big[\log \sqrt{2\pi} + \log [\text{ReLU}\big(W_n^{(3)}s(W_n^{(1)}\vX^T(m))\big)  + e^{W_{n0}^{(3)}}] \\
        &+ \frac{\Big(\vX_n(m) - W_n^{(2)}s\big(W_n^{(1)}\vX^T(m)\big)\Big)^2}{ 2\Big(\text{ReLU}\big(W_n^{(3)}s(W_n^{(1)}\vX^T(m))\big) + e^{W_{n0}^{(3)}}\Big)^2 }\Big]
    \end{split}
    $}
\end{equation}}
where $s(\cdot)$ is the sigmoid function.
If we adopt the simplified notation, Eq.~\eqref{eq_nll_detailed} can also be written as $\mathcal{L}_{nll} (\bm X, A, B) = \sum_{m=1}^M\sum_{n=1}^N[\log(\sigma_n(m)\sqrt{2\pi}) + \frac{(X_n(m) - f_n(\vX(m), A_n))^2}{2\sigma_n^2(m)}]$. Hence we also denote the our loss function as $\mathcal{L}_{nll}(\bm X, A, \bm \sigma^2)$ in the following sections.

\section{The Property of Proposed Two-phase Algorithm}
\subsection{Marginal Log-likelihood}
Our novel formulation based on the general form of SEM in \textbf{Definition 1} introduces the modeling and estimation of noise variances $\bm \sigma^2$, which inevitably increases the optimization difficulties. Not only there are additional variances $\bm \sigma^2$ to learn, but the variances $\bm \sigma^2$ can compromise the accuracy of structural parameters $A$ during training. The interplay between $A$ and $\bm \sigma^2$ can lead the algorithm to converge to unexpected stationary solutions. To alleviate such an issue, we employ a two-phase optimization procedure, where we treat the noise variances $\bm \sigma^2$ as the unknown variables that need to be estimated simultaneously. The detailed derivations are shown in Eq.~\eqref{eq_EM}. 
\begin{equation}\label{eq_EM}
    \scalebox{0.8}{$
    \begin{split}
        \log p(\vX|A) &= \log \int_{\bm {\sigma}^2} p(\vX, \bm \sigma^2|A)\diff \bm \sigma^2 \\
        &= \log \int_{\bm \sigma^2} q(\bm \sigma^2|\vX, \Theta_q)\frac{p(\vX, \bm \sigma^2|A)}{q(\bm \sigma^2|\vX, \Theta_q)} \diff \bm \sigma^2 \\
        &\geq \int_{\bm \sigma^2} q(\bm\sigma^2|\vX, \Theta_q)\log \frac{p(\vX, \bm \sigma^2|A)}{q(\bm \sigma^2|\vX, \Theta_q)}\diff \bm \sigma^2 
    \end{split}
    $}
\end{equation}
\subsection{Convergence Guarantee}\label{sec_EM_Convergence_gaurantee}
Our algorithm can only guarantee achieving a solution with derivatives of the likelihood being arbitrarily close to zero, i.e., the solution is a stationary point. We can view our Phase-I step (Eq. (15)
in the main paper) as an attempt to construct a function $Q(A|A^{t-1}) = \log p(\vX|\hat{\bm \sigma}^2, A)$ by finding the optimal values $\hat{\bm \sigma}^2$. Then Phase-II step can be seen as choosing the $A^t$ to be any value in the set of $A$, denoted as $\Omega$, which maximizes $Q(A|A^{t-1})$. Assume the optimal value of $A$ in $t^{th}$ iteration is selected from $\mathcal{M}(A^{t-1})$, where $\mathcal{M}(\cdot)$ a point-to-set map such that
\[Q(A'|A^{t-1}) \geq Q(A^{t-1}|A^{t-1}) \forall A'\in \mathcal{M}(A^{t-1})\]

We define the log-likelihood in Eq.~(\ref{eq_EM}) as $L(A)$, i.e., $L(A) = \log p(\bm X|A) = -\mathcal{L}_{nll}(\vX, A, \bm \sigma^2)$. According to the Theorem 1 in \citep{wu1983convergence}, if 1) $\mathcal{M}$ is a closed point-to-set map in the complement of $\mathcal{F}$, which is a set of stationary points in the interior of $\Omega$, 2) $L(A^t)>L(A^{t-1})$ for all $A^{t-1}\notin \mathcal{F}$, then the limit points of $\{A^t\}_{t=1}^{T}$ are stationary points of $L$ and $L(A^t)$ converges monotonically to $L^* = L(A^*)$ for a stationary point $A^*$. \citet{wu1983convergence} also gives a sufficient condition for the closedness of $\mathcal{M}$: $Q(\psi|\phi)$ is continuous in both $\psi$ and $\phi$. Such a condition is easily satisfied since our $Q$ is continuous in $A$. For 2), proving $L(A^t) > L(A^{t-1})$ is equivalent to prove $\mathcal{L}_{nll}(\vX, A^{t-1}, \big(\bm \sigma^2\big)^{t-1}) > \mathcal{L}_{nll}(\vX, A^{t}, \big(\bm \sigma^2\big)^{t}) $. 
$\big(\bm \sigma^2\big)^t$ is obtained by finding the variance parameters $B^t$ through minimizing $\mathcal{L}_{nll}(\vX, A^{t-1}, B)$, hence we have
\begin{equation}\label{eq_ineq1}
    \mathcal{L}_{nll}(\vX, A^{t-1}, \big(\bm \sigma^2\big)^t) \leq \mathcal{L}_{nll}(\vX, A^{t-1}, \big(\bm \sigma^2\big)^{t-1})
\end{equation}
$A^t$ is obtained by minimizing $\mathcal{L}_{nll}(\vX, A, \big(\bm \sigma^2\big)^t)$, and we have
\begin{equation}\label{eq_ineq2}
    \mathcal{L}_{nll}(\vX, A^t, \big(\bm \sigma^2\big)^t) \leq \mathcal{L}_{nll}(\vX, A^{t-1}, \big(\bm \sigma^2\big)^t)
\end{equation}
Combine Eq.~(\ref{eq_ineq1}) and Eq.~(\ref{eq_ineq2}), we have
\begin{equation}
    \begin{split}
        \mathcal{L}_{nll}(\vX, A^{t-1}, \big(\bm \sigma^2\big)^{t-1}) &\geq \mathcal{L}_{nll}(\vX, A^{t-1}, \big(\bm \sigma^2\big)^t)\\
        &\geq \mathcal{L}_{nll}(\vX, A^t, \big(\bm \sigma^2\big)^t)
    \end{split}
\end{equation}

Note that the equalities can not be satisfied simultaneously, otherwise, the algorithm converges at iteration $t-1$ and $A^{t-1} \in \mathcal{F}$. Therefore, 
\begin{equation}
    \mathcal{L}_{nll}(\vX, A^{t-1}, \big(\bm \sigma^2\big)^{t-1}) > \mathcal{L}_{nll}(\vX, A^t, \big(\bm \sigma^2\big)^t)
\end{equation}
Since our proposed algorithm satisfies both 1) and 2), it can converge monotonically to an optimal $L^*$ with a stationary point $A^*$. 

\section{Identifiable multivariate HNM}\label{proof_identifiability}
\textbf{Theorem 3.2} in the main paper provides the relaxing and implementable sufficient conditions for multivariate HNM defined in \textbf{Definition 3.1}. We provide a sketch for the proof of \textbf{Theorem 3.2} and separate the proof into two components: (1) For bivariate case on variables $X_1$ and $X_2$, identify the assumptions that must hold so the HNMs for causal relations $X_1 \leftarrow X_2$ and $X_1 \rightarrow X_2$ both satisfy the given distribution $p(X_1, X_2)$. \textbf{We then identify the sufficient conditions that violate the assumptions from (1). If those conditions are satisfied, then only the HNM for the true causal direction satisfies the data distribution. Then HNM is identifiable subject to those conditions for the bivariate cases.}  (2) We follow the standard approach to prove the identifiability for multivariate cases using the identifiability theorem for bivariate cases. 


First, we {propose \textbf{Lemma \ref{lem: ident_bivariate}} to prove the identifiability of the HNM for bivariate cases subject to certain conditions.
We aim to identify conditions that are easy to model and estimate in practice.
\begin{lemma}\label{lem: ident_bivariate}
    Assume a random set with two variables $X = (X_1, X_2)$ follows the HNM described by Eq. (5)
    , with $E_1, E_2$ be the independent exogenous noise variables for $X_1, X_2$. $E_1, E_2$ follow Gaussian distributions. If  functions $f_j, \sigma_j$ linking cause to effect satisfy 1) $f_j$ is nonlinear, and 2) $\sigma_j$ is a piece-wise function, then the HNM is identifiable.
\end{lemma}
\begin{proof}
For variables $X_1, X_2$, assume they follow the model
\begin{equation}\label{eq_bi_forward}
    X_2 = f_2(X_1) + \sigma_2(X_1)E_2
\end{equation}
where $E_2$ is a standard Gaussian distribution, $f_2, \sigma_2$ are twice-differentiable scalar functions and $\sigma_2(X_1)>0$. If a backward model exists, i,e, the data also follows the same model in the other direction, 
\begin{equation}\label{eq_bi_backward}
    X_1 = f_1(X_2)+\sigma_1(X_2)E_1
\end{equation}
where $E_1$ is a standard Gaussian distribution, $f_1, \sigma_1$ are twice-differentiable scalar functions and $\sigma_1(X_2)>0$. The assumptions that must hold so the forward and backward models co-exist have been studied and identified by \citet{khemakhem2021causal} and \citet{immer2022identifiability}. We employ theoretical results from \textbf{Theorem 2} in \citet{khemakhem2021causal}.  If Eq.~\eqref{eq_bi_forward} and Eq.~\eqref{eq_bi_backward} co-exist, then one of the following scenarios must hold: (1) $(\sigma_2, f_2) = (\frac{1}{Q}, \frac{P}{Q})$ and $(\sigma_1, f_1) = (\frac{1}{Q'}, \frac{P'}{Q'})$ where $Q, Q'$ are polynomials of degree two, $Q, Q'>0$, $P, P'$ are polynomials of degree two or less, and $p(X_1), p(X_2)$ are strictly log-mix-rational-log. (2) $\sigma_1, \sigma_2$ are constant, $f_1, f_2$ are linear and $p(X_1), p(X_2)$ are Gaussian densities.

By making $f_j$ to be nonlinear, scenario (2) does not apply to our HNM. We then choose $\sigma_j$ not to be in the format of $\frac{1}{Q}, Q$ are polynomials of degree two. Hence we let $\sigma_j$ be a piece-wise function. For example, in our formulation, for causal direction $X_1 \rightarrow X_2$, we choose $\sigma_2(X_1) = \text{ReLU}\big(w_3s(w_1X_1)\big) + e^{w_{30}}$, where $e^{w_{30}}$ is to make sure $\sigma_2(X_1) >0$, $s(\cdot)$ is sigmoid activation function. If conditions on $f_j, \sigma_j$ are satisfied, then both scenarios do not hold for our HNM. There is no backward model for any distribution that satisfies Eq. (5)
for bivariate cases. Hence, the model is identifiable.
\end{proof}
Compared to our identifiable conditions in \textbf{Lemma \ref{lem: ident_bivariate}}, conditions in \citet{khemakhem2021causal} ensure scenarios (1) and (2) do not hold by choosing $f_j$ to be nonlinear and invertible. However, the invertibility condition cannot be readily adapted to multivariate cases due to the different number of dimensions between inputs and outputs of function $f_j$ s. Our piecewise $\sigma$ conditions are more relaxed and implementable on multivariate cases.



We then prove the identifiability of multivariate HNM using \textbf{Lemma \ref{lem: ident_bivariate}}. To prove that our HNM is identifiable for multivariate cases is to prove that a unique graph $\mathcal{G}$ can be identified subject to HNM. In the following proof, we employ the \textbf{Proposition 28}, \textbf{Lemma 35}, and \textbf{Lemma 36} from \citet{peters2014causal}. We show the theoretical results from those proposition and lemmas in the format of our HNM. 

Assume that there exists another HNM with graph $\mathcal{G}'$ that $\mathcal{G} \neq \mathcal{G}'$. According to the \textbf{Propostion 28} in \citet{peters2014causal}, let $\mathcal{G}$ and $\mathcal{G}'$ be two different DAGs over a set of variables $\bm X$. Assume $p(\bm X)$ is generated by our HNM and satisfies the Markov condition and causal minimality with respect to $\mathcal{G}$ and $\mathbf{G}'$. Then there are variables $L, Y \in \bm X$ such that for the set $\bm Q :=\textbf{PA}_{Y}^\mathcal{G}\backslash \{L\}, \bm R := \textbf{PA}_{L}^{\mathcal{G}'}\backslash \{Y\}$ and $\bm S := \bm Q \cup \bm R$, we have: A) $L \rightarrow Y$ in $\mathcal{G}$ and $Y \rightarrow L$ in $\mathcal{G}'$. B) $\bm S \subseteq \textbf{ND}_{Y}^{\mathcal{G}}\backslash \{L\}$ and $\bm S \subseteq \textbf{ND}_{L}^{\mathcal{G}'}\backslash \{Y\}$. $\textbf{PA}_Y^{\mathcal{G}}$ is the set of parent variables of $Y$ in graph $\mathcal{G}$. $\textbf{ND}_{Y}^{\mathcal{G}}$ is the set of non-descendant variables of $Y$ in graph $\mathcal{G}'$.

We consider $\bm S = \bm s$ with $p(\bm s)>0$. Denote $L^{*} := L|\bm S = \bm s$ and $Y^{*} := Y|\bm S = \bm s$. \textbf{Lemma 36} in \cite{peters2014causal} states that if $p(\bm X)$ is generated according to the SEM models in Eq.~\eqref{eq_general_SEM}: 
\begin{equation}\label{eq_general_SEM}
    X_n = g_n(X_{\pi_n}, E_n), n = 1, 2, \cdots, N, X_n \in \bm X
\end{equation}
with corresponding DAG $\mathcal{G}$, then for a variable $X_n\in \bm X$, if $\bm K \subseteq \textbf{ND}_{X_n}^{\mathcal{G}}$ then $E_{X_n} \indep \bm K$. Our HNM can be viewed one specific class of the SEM in Eq.~\eqref{eq_general_SEM}. Hence, \textbf{Lemma 36} holds under our HNM and renders $E_Y \indep (L, \bm S)$ and $E_L \indep (Y, \bm S)$ . 

\textbf{Lemma 35} from \citep{peters2014causal} indicates that if $E_Y \indep (Y, \bm Q, \bm R)$ then for all $\bm q, \bm r$ with $p(\bm q, \bm r) > 0$, $g(Y, \bm Q, E_Y)|_{\bm Q = \bm q, \bm R = \bm r} = g(Y|_{\bm Q=\bm q, \bm R = \bm r}, \bm q, E_Y)$. We apply \textbf{Lemma 35} and obtain that 
\begin{equation}\label{eq_g_bi_L}
\scalebox{0.9}{$
    g(L, \bm Q, E_Y)|_{\bm S = \bm s} = g(L|_{\bm S = \bm s}, \bm q, E_Y) = g(L^{*}, \bm q, E_Y)
    $}
\end{equation}
\begin{equation}\label{eq_g_bi_Y}
\scalebox{0.9}{$
    g(Y, \bm R, E_L)|_{\bm S = \bm s} = g(Y|_{\bm S = \bm s}, \bm r, E_L) = g(Y^{*}, \bm r, E_L)
    $}
\end{equation}
Hence according to our definition, we have,
\begin{equation}\label{eq_d1}
    Y^* = f_Y(\bm q, L^{*}) + \sigma_Y(\bm q, L^{*})E_Y, E_Y \indep L^{*} \text{ in } \mathcal{G}
\end{equation}
\begin{equation}\label{eq_d2}
    L^* = f_L(\bm r, Y^{*}) + \sigma_X(\bm r, Y^{*})E_L, E_L \indep Y^{*} \text{ in } \mathcal{G}'
\end{equation}
However, the co-existence of both Eq.~\eqref{eq_d1} and Eq.~\eqref{eq_d2} contradicts our identifiability theorem for the bivariate cases. Therefore, the assumption that there exists another HNM with graph $\mathcal{G}'$ that
$\mathcal{G} = \mathcal{G}'$ does not hold. Only one unique DAG $\mathcal{G}$ can be identified from $p(\bm X)$.
\section{Generality of the proposed method}\label{derivations_generality}
\subsection{Comparison with SoA Methods using Reconstruction Loss }
Many existing methods~\citep{zheng2018dags, zheng2020learning, yu2019dag} adopt the reconstruction loss as the optimization objective, usually based on the SEM with additive noise. The score of the DAG learning problem, i.e., $F(A, \vX)$ in Eq. (3)
of the main paper
can be calculated through Eq.~(\ref{eq_rec_loss}).
\begin{equation}\label{eq_rec_loss}
    \scalebox{0.9}{$
        \begin{split}
            F(A, \bm X) &= \frac{1}{2M} \|\bm X - f(\bm X,A)\|_F^2 \\
            &= \frac{1}{2M} \sum_{m=1}^M \sum_{n=1}^N \big(X_n(m) - f_n(\bm X(m), A_n)\big)^2
        \end{split}
        $}
\end{equation}
 Substituting $\sigma_n^2(m) = \sigma^2$ into our training objective in Eq.~\eqref{eq_nll_detailed}, we can obtain
\begin{equation}\label{eq8}
    \scalebox{0.8}{$
    \begin{split}
       & \mathcal{L}_{nll}(\bm X, A, \sigma^2) \\ 
       &= MN\log(\sqrt{2\pi \sigma^2}) +  \frac{\sum_{m=1}^M \sum_{n=1}^N \big(X_n(m) - f_n(\bm X(m), A_n)\big)^2}{2\sigma^2} \\
       &= \frac{\sum_{m=1}^M \sum_{n=1}^N \big(X_n(m) - f_n(\bm X(m), A_n)\big)^2}{2\sigma^2} + \text{const}
    \end{split}
    $}
\end{equation}
With constant variance $\sigma^2$,  optimizing the NLL loss in Eq.~(\ref{eq8}) is equivalent to optimizing the reconstruction loss in Eq.~(\ref{eq_rec_loss}).

\subsection{Comparison with SoA Methods Using Likelihood Loss under SEM with Additive Noise.}
GOLEM-NV~\cite{ng2020role} and GraN-DAG~\citep{lachapelle2019gradient} relax the equal noise variance assumption across variables under SEM with additive noise. Substitute $\sigma^2_n(m) = \sigma^2_n$ into Eq.~(\ref{eq_nll_detailed}), we obtain $\mathcal{L}_{nll}(\bm X, A, \bm \sigma^2)$ as:
\begin{equation}\label{eq9}
    \scalebox{0.8}{$
    \begin{split}
       & \mathcal{L}_{nll}(\bm X, A, \bm \sigma^2) \\
       =&  M\sum_{n=1}^N\log(\sqrt{2\pi \sigma^2_n}) + \sum_{n=1}^N \frac{\sum_{m=1}^M \big(X_n(m) - f_n(\bm X(m), A_n)\big)^2}{2\sigma^2_n}
    \end{split}
    $}
\end{equation}
Eq.~\eqref{eq9} is equivalent to the loss in \citet{lachapelle2019gradient}. $\bm \sigma^2 = (\sigma^2_1, \sigma^2_2, \cdots, \sigma^2_N)$ are treated as parameters and estimated during training. However, if we choose causal function $f_n(\cdot)$ in Eq.~\eqref{eq9} as a linear function, i.e., $f_n(\bm X(m), A_n) = \bm X(m)A_n$, then we can show that our derived loss is equivalent to the loss derived in \citet{ng2020role}. 
As shown in \citet{ng2020role}, when considering $\mathcal{L}_{nll}$ as a function of $\sigma_n^2$, its local extreme values occur when $\frac{\partial \mathcal{L}_{nll}}{\partial \sigma^2_n} = 0$, i.e.,
\begin{equation}\label{eq10}
    \hat{\sigma}^2_n = \frac{\sum_{m=1}^M\Big[\big(X_n(m) - \bm X(m)A_n\big)^2\Big]}{M}
\end{equation}
Substitute Eq.~(\ref{eq10}) into Eq.~(\ref{eq9}), we obtain %
$\mathcal{L}_{nll}$ as:
\begin{equation}\label{eq11}
    \scalebox{0.75}{$
    \begin{split}
     & \mathcal{L}_{nll} \\
     =& \frac{MN}{2}\Big(1+\log(2\pi) - \log(M)\Big) + \frac{M}{2}\sum_{n=1}^N\Big[  \log\sum_{m=1}^M\big(X_n(m) - \bm X(m)A_n\big)^2\Big]\\
       =& \frac{M}{2}\sum_{n=1}^N\Big[\log\sum_{m=1}^M\big(X_n(m) - \bm X(m)A_n\big)^2\Big] + \text{const}
    \end{split}
    $}
\end{equation}
Eq.~(\ref{eq11}) is equivalent to the likelihood-based objective in Appendix C.1 of~\citet{ng2020role} under the assumption that $A$ satisfies acyclicity constraint.

Hence, the losses that are derived under the additive noise SEM are merely special cases for the losses derived under more general SEM with affine noise.
\section{The Iterative DAG Learning Method}
\subsection{Procedures of the propose algorithm}\label{pseudo-code}
{We summarized the main procedure of our proposed two-phase iterative DAG learning algorithm in Algorithm \ref{algorithm1}, with Phase-I procedure in Algorithm \ref{algorithm2} and Phase-II procedure in Algorithm \ref{algorithm3}.}
\input{algorithm1}
\input{algorithm2}
\input{algorithm3}
\section{Dataset Description}\label{sec_data_description}


\subsection{Synthetic Data}
To validate the effectiveness of various types of datasets, we apply our proposed algorithms to synthetic data, where various levels of noise heterosedacity are incorporated during the generation process. We adopt the standard setup in \cite{zheng2020learning, yu2019dag, lachapelle2019gradient}. The ground-truth DAGs are generated from Erdo-Renyi (ER) 
with $k$ expected edges, which we set as 1 and 2. We generate 10 graphs for each graph setting
with different numbers of variables $d = 5, 10, 20, 50$. For each setting, we simulate 10 trials with $n=1000$ data observations.

\paragraph{Synthetic homoscedastic noise data.}
We first conduct experiments on nonlinear synthetic homoscedastic noise data. 
We consider two types of homoscedastic data. The simpler version assumes that the noise corresponding to different variables across data samples has equal variance, i.e. noises are homoscedastic w.r.t both variables and observations. The model formulation of \citet{zheng2020learning} satisfies such a data generation process and we employ similar procedures to generate nonlinear data with Gaussian noise. We denote such type as \textbf{homo-EV} data. Given a randomly generated binary DAG $\mathcal{G}$, the observations are sampled from the SEMs in Eq.~\eqref{eq_data_generate_homo} following the topological order induced by $\mathcal{G}$ :
\begin{equation}\label{eq_data_generate_homo}
    X_n = f_n(X_{pa(n)}) + Z_n, n = 1, 2, \cdots, N
\end{equation}
where we chose $f_n(\cdot)$ to be randomly initialized MLPs with one hidden layer of size 100 and sigmoid activation. \textbf{$Z_n$ are standard Gaussian noises, i.e., }$Z_n\sim \mathcal{N}(0, 1)$.
The slighter complex version, denoted as \textbf{homo-NV} data, allows the noises for different variables to have non-equal variances yet the noise variances across observations remain to be the same, i.e., $Z_n\sim \mathcal{N}(0, \sigma^2_n), n = 1,2,\cdots, N$. We obtain the variances by sampling from a uniform distribution, i.e., $\sigma^2_n \sim U[0.5, 2]$. We employ a similar data generation process with \citet{lachapelle2019gradient} since its formulation fits assumptions.

\paragraph{Synthetic heteroscedastic noise data.} We then evaluate our proposed algorithm on nonlinear synthetic heteroscedastic noise data. 
For heteroscedastic noise data, the noise variances vary across both variables and observations. Hence, heteroscedastic noise data is more challenging to accurately recover the DAG from the given observations. Given a random directed acyclic graph $\mathcal{G}$ with binary entries, we generate observations from the SEMs in Eq.~\eqref{eq_data_generate_hetero} following the topological order induced by $\mathcal{G}$:
\begin{equation}\label{eq_data_generate_hetero}
    X_n = f_n(X_{pa(n)}) + e^{g_n(X_{pa(n)})}Z_n, n = 1, 2, \cdots, N
\end{equation}
$f_n(\cdot)$ and $g_n(\cdot)$ are chosen to be randomly initialized MLPs with one hidden layer of size 100 and sigmoid activation. {During the data generation process, we choose the variance function to be a global estimator, i.e., $\sigma_n = e^{g_n(\bm X_{pa(n)})}$ in order to test our piece-wise variance function's ability in recovering accurate variances.} $Z_n$ are standard Gaussian noises, $Z_n\sim\mathcal{N}(0, 1)$. 
We denote the data generated through the above process as \textbf{hetero data}.


\subsection{Real Data. }
To demonstrate the effectiveness on real data, we test the proposed method on two widely-studied real benchmark datasets: the Sachs dataset~\cite{sachs2005causal} and the cause-effect pairs dataset~\cite{sgouritsa2015inference}. The Sachs dataset contains real-world flow cytometry data from for modeling protein signaling pathways. The dataset comprises continuous measurements of 11 phosphoproteins in individual T-cells. We specifically selected 853 observations corresponding to the first experimental condition outlined in \cite{sachs2005causal} as our dataset $\mathcal{D}$. For our reference graph (ground truth), we utilize the provided DAG, which consists of 11 nodes and 17 edges. It is important to note that this consensus graph may not provide a comprehensive or entirely accurate representation of the system under study. The cause-effect pairs dataset provides {99} sets of data with given cause-effect relations between variables.

\section{Experiment Setting}\label{eval_metrics}
\paragraph{Evaluation metrics.} We employ 3 evaluation metrics to evaluate the accuracy of DAG learning: SHD, auSHDC, auPRC.
\paragraph{SHD:} SHD is the most widely used evaluation metric to evaluate the accuracy of a learned graph. However, a heuristic threshold approach must be performed to infer a DAG $\mathcal{G}$ from the weighted adjacency matrix. We report two SHDs in this paper: the SHD with a threshold of $0.3$, which is also chosen by the majority of the existing methods, and the minimum SHD obtained by using thresholds within the chosen range. 


\textbf{We also choose evaluation metrics that are less susceptible to thresholding. } 

\paragraph{auSHDC: }To reduce the effect of thresholding on SHD, we choose a reasonable range for thresholds, estimate the SHD value of the graph thresholded with different thresholds, and plot the curve of SHDs versus thresholds.  We employ the area under the SHD curve as a measurement of graph accuracy. A small auSHDC value indicates that the applied algorithm performs well and is robust regardless of the thresholds. Since the synthetic graph parameters are from $U([-2.0, 0.5]\cup[0.5, 2.0])$, we believe $[0.2, 0.75]$ is a reasonable range for synthetic datasets. We adjust the range to be $[0.25, 0.75]$ for large models based on the empirical results.

\paragraph{auPRC: } auPRC does not require to choose a constant value as a threshold. The precision-recall curve (PRC) can be plotted from the learned weighted adjacency matrix. The accuracy performance can be reflected by the area under the precision-recall curve (auPRC). The graph with larger auPRC values has a higher accuracy.

\subsection{Implementation Details}
We implemented the algorithm following the pseudo-code outlined in Algorithm \ref{algorithm1}, \ref{algorithm2}, and \ref{algorithm3} in the main paper. We choose the LBFGS optimizer from the \texttt{scipy} library. For hyper-parameters in Algorithm 3, we set $\epsilon = 10^{-8}, c = 0.25, s = 10$ as suggested in the \citet{zheng2018dags} where the augmented Lagrangian process for DAG learning is first introduced. We set the number of hidden neurons as $m_1 =10$ for all the baselines. 
We conducted all experiments on a workstation with a 3.1 GHz CPU.

\section{Detailed empirical results}\label{detailed_results}
\subsection{Nonlinear Synthetic Data}
We show the detailed empirical results for homo-ev, homo-nv, and hetero data
in Table \ref{tb:homo-ev}, \ref{tb:homo-nv}, and \ref{tb:hetero} respectively. We compared to baselines on different types of data depending on the matchness of their underlying model assumptions. However, since the GOLEM-EV and GOLEM-NV are implemented for data with linear relations, hence the results on nonlinear data are worse than other nonlinear DAG learning methods. In conclusion, we expect, and observed from three tables that our proposed method can achieve comparable results on data with equal noise variances across observations. Our proposed method outperforms the baselines on heteroscedastic data whereby the noise variances also vary with different values of causes.

We also performed experiments on larger datasets with $N = 50$ variables. Based on the data generation process that we elaborate on above, a significant degree of noise has been embedded into the data, causing compromised performance on both our method and baseline methods. However, we will probably never expect such an amount of data noise in real-world applications. 
\input{tb-homo-ev}
\input{tb-homo-nv}
\input{tb-hetero}

\subsection{Comprehensive Comparison}\label{comprehnsive_comparision}
\begin{table}[H]
    \centering
    \small
    \begin{tabular}{|c||c|c|c|}
        \hline 
        \multirow{2}{*}{Methods} 
        & \multicolumn{3}{c|}{ER1} \\
        & d5 & d10 & d20 \\
        \hline 
        HOST & 
        $\bf{2.5\pm1.43}$ & $\bf{8.2\pm 2.36}$ & $23.9\pm 6.93$ \\
        sortnregress 
        & $4.0\pm 2.19$ & $13.9\pm 4.13$ & $46.7\pm8.32$ 
        \\
        CD-NOD 
        & $7.1\pm1.60$ & $18.6\pm5.48$ & $30.6\pm 7.03$ 
        \\
        \rowcolor{Gray}
        ICDH 
        & $4.0\pm 1.58$ & $11.4\pm 4.93$ & $\bf{23.6\pm9.34}$ 
        \\
        \hline 
        \hline 
        \multirow{2}{*}{Methods} 
        & \multicolumn{3}{c|}{ER2}\\
        & d5 & d10 & d20 \\
        \hline 
        HOST &
        $6.3\pm 1.90$ & $15.7\pm4.75$ & $45.1\pm 7.48$ \\
        sortnregress 
        & $7.2\pm1.40$ & $19.8\pm4.47$ & $44.3\pm4.50$ \\
        CD-NOD 
        & $5.2\pm1.45$ & $24.5\pm3.36$ & $63.7\pm7.60$ \\
        \rowcolor{Gray}
        ICDH 
        & $\bf{4.2\pm3.06}$ & $\bf{15.0\pm3.05}$ & $\bf{40.0\pm4.40}$\\
        \hline 
    \end{tabular}
    \caption{Comparison to methods under different SEMs.}
    \label{tab:comprehensive}
    \vspace{-5 mm}
\end{table}
The methods designed for heterogeneous data (CD-NOD), and the scale-invariant data (sortnregress) use quite different assumptions from heteroscedastic noise data (ours).  The heteroscedastic noise may render sortnregress and CD-NOD to wrongly estimate the marginal variance and/or identify the wrong causal order. We present the performance of methods that satisfy assumptions for heteroscedastic noise data, i.e., HOST and Our ICDH against methods that are developed under different assumptions, i.e., CD-NOD and sortnregress on heteroscedastic noise data in Table \ref{tab:comprehensive}. The empirical results on hetero data show ICDH and HOST perform better than CD-NOD and sortnregress. Since HOST and ICDH adopt different optimization approaches, ICDH performs better on denser graphs than HOST. Please see the results in Appendix \ref{dense_graphs}.
\subsection{Performance on Denser Graphs}\label{dense_graphs}
We experimented on hetero data generated from ER3 d10 graphs and evaluated the accuracy via SHD and SID. 
\begin{table}[hpt]
    \centering
    \normalsize
    \begin{tabular}{|c|c|c|}
        \hline 
        Methods & SHD (0.3) & SID \\
        \hline 
        NOTEARS & 20.7 & 54.7 \\
        HOST & 16.6 & 42.2 \\
        GFBS & 19.3 & 54.3 \\
        GES & 34.0 & 71.5 \\
        \rowcolor{Gray}
        ICDH (ours) & \bf{14.1} & \bf{39.5} \\
        \hline 
    \end{tabular}
    \caption{Empirical results on ER3 d10 Hetero noise data.}
    \label{tab:dense_graph}
\end{table}

Our ICDH method achieves the optimal SHD ($14.1$) and SID ($39.5$), outperforms the essential state-of-the-art baselines. In particular, combined with results from Table \ref{tab:comprehensive}, our ICDH outperforms HOST by a larger margin on larger and denser graphs.
\section{Limitations}
{In this paper, we 
propose a novel DAG learning formulation based on a general SEM which allows the modeling of the variation of noise variances across both variables and observations. To solve the increasing difficulties in optimization, we propose a two-phase iterative learning algorithm. However, there are two main limitations to the proposed algorithm. 
First, the proposed algorithm inevitably inherits the typical optimization difficulties for iterative optimization algorithms. The proposed iterative DAG learning algorithm only guarantees convergence to a stationary solution. Hence good initialization is crucial for the algorithm to achieve satisfactory performance. Another limitation is that our formulation has to satisfy the definition of HNM in \textbf{Definition 1
} and is identifiable only when sufficient conditions in \textbf{Theorem \ref{theorem: identifiable}
} are satisfied.

}

\end{document}

%% file: algorithm1.tex
\begin{algorithm}[hpt]
    \centering
    \begin{algorithmic}[1]
        \STATE \textbf{Input:} Data $\vX$ 
        \STATE  \textbf{Output:} $(W^{(1)})^*, (W^{(2)})^*, (W^{(3)})^*$
        \STATE  Initial $(W^{(1)})^0, (W^{(2)})^0, (W^{(3)})^0$ with $0$
        \STATE  $\big(\hat{\bm \sigma}\big)^0 \leftarrow \sigma\big(\vX, (W^{(1)})^0, (W^{(3)})^0\big)$ 
        \STATE  $(W^{(1)})^1, (W^{(2)})^1 \leftarrow \texttt{Phase-II-Update}(\vX, (\hat{\bm \sigma})^0)$
        \STATE  $t\leftarrow 0$
        \REPEAT 
        \STATE  $t\leftarrow t+1$
        \STATE  $\{\text{Phase-I step:}\}$
        \\STATE  $(W^{(3)})^t = \texttt{Phase-I-Update}\big(\vX, (W^{(1)})^{t-1}, (W^{(2)})^{t-1}\big)$
        \STATE  $\big(\bm \hat{\sigma}\big)^t \leftarrow \sigma\big(\vX, (W^{(1)})^{t-1}, (W^{(3)})^{t}\big)$
        \STATE  $\{\text{Phase-II step:}\}$
        \STATE  $(W^{(1)})^t, (W^{(2)})^t = \texttt{Phase-II-Update}\big(\vX,(\hat{\bm \sigma})^t \big)$
        \UNTIL{Converge}
        \STATE  $(W^{(1)})^*, (W^{(2)})^*, (W^{(3)})^* \leftarrow (W^{(1)})^t, (W^{(2)})^t, (W^{(3)})^t$
        \STATE  \textbf{Return} $(W^{(1)})^*, (W^{(2)})^*, (W^{(3)})^*$
    \end{algorithmic}
    \caption{Main Procedure}\label{algorithm1}
\end{algorithm}

%% file: algorithm2.tex
\begin{algorithm}[hpt]
    \centering
    \caption{Phase-I-Update Procedure}\label{algorithm2}
    \begin{algorithmic}[1]
        \STATE  \textbf{Input:} Data $\vX$, $\hat{W}^{(1)}$, $\hat{W}^{(2)}$
        \STATE  \textbf{Output: }$(W^{(3)})^*$
        \STATE  Initial $(W^{(3)})^0$ with small values
        \STATE  $ (W^{(3)})^* = \arg\min_{W^{(3)}}\mathcal{L}_{nll}(\vX, \hat{W}^{(1)}, \hat{W}^{(2)}, W^{(3)})$
        \STATE  \textbf{Return}  $(W^{(3)})^*$
    \end{algorithmic}
\end{algorithm}

%% file: algorithm3.tex
\begin{algorithm}[hpt]
    \centering
    \caption{Phase-II-Update Procedure}\label{algorithm3}
    \begin{algorithmic}[1]
        \STATE  \textbf{Input:} Data $\vX$; Noise variances $\hat{\bm \sigma}^2$ 
        \STATE  \textbf{Output: }$(W^{(1)})^*, (W^{(2)})^*$
        \STATE  Initial $(W^{(1)})^0, (W^{(2)})^0$ with small values.
        \STATE  $\alpha=0, \rho=1, t\leftarrow0$
        \WHILE{$h\Big((W^{(1)})^t\Big) > \epsilon$}
        \WHILE{$\rho < \rho_{max}$}
        \STATE \begin{equation}
            \scalebox{0.7}{$
            \begin{split}
                & (W^{(1)})^c , (W^{(2)})^c \\
            & = \arg\min_{W^{(1)}, W^{(2)}}\mathcal{L}_{nll}(\vX, W^{(1)}, W^{(2)}, \hat{\bm \sigma}^2) + \frac{\rho}{2}h^2\Big((W^{(1)})^c\Big) + \alpha h\Big((W^{(1)})^c\Big)
            \end{split}
            $}
        \end{equation}
        \IF{$h\Big((W^{(1)})^c\Big) < c\cdot h\Big((W^{(1)})^t\Big)$}
            \STATE  $(W^{(1)})^{t+1}, (W^{(2)})^{t+1} \leftarrow (W^{(1)})^c, (W^{(2)})^c$
        \ELSE 
            \STATE  $\rho \leftarrow s\cdot \rho$
        \ENDIF 
        \ENDWHILE 
        \STATE  $\alpha \leftarrow \alpha + \rho h\Big((W^{(1)})^{t+1}\Big)$
        \STATE  $t \leftarrow t+1$
        \ENDWHILE 
        \STATE $(W^{(1)})^*, (W^{(2)})^*\leftarrow (W^{(1)})^t, (W^{(2)})^t$
        \STATE  \textbf{Return}  $(W^{(1)})^*, (W^{(2)})^*$
    \end{algorithmic}
\end{algorithm}

%% file: tb-homo-ev.tex
\begin{table*}[hpt]
\centering
\scriptsize
\scalebox{0.95}{
\begin{tabular}{cc|ccc|ccc|ccc}
\hline
\hline
 &&\multicolumn{3}{c}{\underline{NOTEARS-MLP}} & \multicolumn{3}{c}{\underline{GOLEM-EV}} & \multicolumn{3}{c}{\underline{ICDH(ours)}} \\
graph & d & auSHDC & SHD & auPRC & auSHDC & SHD & auPRC & auSHDC & SHD & auPRC \\
\hline
&$5$& $0.68\pm0.15$ & $1.0\pm0.20$ & $0.64\pm0.03$ & $1.85 \pm 0.82$ & $3.5\pm1.55$ & $0.38\pm0.03$ & $0.69\pm0.39$ & $1.0\pm 0.20$ & $0.64\pm0.03$ \\
ER1 &$10$ & $1.86\pm0.90$ & $2.2 \pm 1.81$ & $0.70\pm0.05$ & $5.56\pm1.61$ & $10.8\pm2.81$ & $0.27\pm0.01$ & $1.44\pm0.52$ & $2.0\pm1.83$ & $0.71\pm0.02$ \\
&$20$ & $3.22\pm2.70$ & $6.2\pm2.77$ & $0.78\pm0.06$ & $14.92\pm5.42$ & $28.4\pm4.27$ & $0.13\pm0.01$ & $2.61\pm1.39$ & $5.2\pm2.49$ & $0.82\pm0.13$\\
&$50$&  - & $24.5\pm 6.20$ & - & - & $50.6\pm8.4$ & - & - & $22.5\pm5.5$ & - \\
\hline
 &&\multicolumn{3}{c}{\underline{NOTEARS-MLP}} & \multicolumn{3}{c}{\underline{GOLEM-EV}} & \multicolumn{3}{c}{\underline{ICDH(ours)}} \\
graph & d & auSHDC & SHD & auPRC & auSHDC & SHD & auPRC & auSHDC & SHD & auPRC \\
 \hline
&$5$ &  $1.04\pm0.61$ & $1.80\pm0.91$ & $0.78\pm0.07$ & $2.28 \pm 1.54$ & $4.40\pm2.72$ & $0.2234\pm0.01$ & $1.03\pm0.60$ & $1.8\pm0.98$ & $0.78\pm0.07$ \\
ER2 & $10$ & $2.35\pm0.89$ & $3.4\pm1.74$ &$0.87\pm0.06$ & $8.91\pm2.17$ & $17.2\pm6.07$ & $0.19\pm0.02$ & $2.31\pm0.92$ & $3.2\pm1.72$ & $0.87\pm0.05$ \\
& $20$ & $7.98\pm0.93$ & $14.0\pm2.00$ & $0.68\pm0.07$ & $23.17\pm12.88$ & $45.8\pm6.39$ & $0.08\pm0.03$ & $5.13\pm1.74$ & $8.4\pm3.98$ & $0.84\pm0.08$  \\
&$50$&  - & $39.0\pm 9.7$ & - & - & $100.6\pm8.4$ & - & - & $30.9\pm10.3$ & - \\
\hline
\end{tabular}
}
\caption{Comparison of all baseline algorithms on nonlinear synthetic homoscedastic noise datasets with equal variances across variables and observations (homo-EV): results (mean $\pm$ standard deviation over $10$ trails) on auSHDC, SHD, and auPRC. 
}
\label{tb:homo-ev}
\end{table*}

%% file: tb-homo-nv.tex
\begin{table*}[hpt]
\centering
\scriptsize
\scalebox{0.9}{
\begin{tabular}{cc|ccc|ccc|ccc}
\hline
\hline
 &&\multicolumn{3}{c}{\underline{auSHDC}} & \multicolumn{3}{c}{\underline{SHD}} & \multicolumn{3}{c}{\underline{auPRC}} \\
graph & methods & $d5$ & $d10$ & $d20$ & $d5$ & $d10$ & $d20$ & $d5$ & $d10$ & $d20$ \\
\hline
ER1 & NOTEARS-MLP & $0.40\pm0.38$ & $0.53\pm 0.31$ & $2.60\pm1.19$ & $0.4\pm0.80$ & $0.4\pm0.49$ & $4.2\pm1.94$ & $0.72\pm0.10$ & $0.89\pm0.02$ & $0.80\pm0.08$ \\
& GOLEM-NV & $2.14\pm1.57$ & $3.95\pm2.81$ & $10.25\pm2.21$ & $3.20\pm3.71$ & $5.60\pm5.82$ & $15.00\pm5.48$ & $0.85\pm0.13$ & $0.88\pm0.10$ & $0.92\pm0.04$\\
& GraN-DAG & - & - & - & $2.4\pm1.51$ & $3.6\pm1.52$ & $6.2\pm2.77$ & - & - & -\\
& ICDH(ours) & $0.41\pm0.37$ & $0.49\pm0.26$ & $2.53\pm1.19$ & $0.4 \pm 0.80$ & $0.6\pm0.49$ & $4.2\pm1.94$ & $0.71\pm0.11$ & $0.89\pm 0.02$ & $0.79\pm0.09$ \\
\hline
 &&\multicolumn{3}{c}{\underline{auSHDC}} & \multicolumn{3}{c}{\underline{SHD}} & \multicolumn{3}{c}{\underline{auPRC}} \\
graph & methods & $d5$ & $d10$ & $d20$ & $d5$ & $d10$ & $d20$ & $d5$ & $d10$ & $d20$ \\
 \hline
ER2 & NOTEARS-MLP & $0.55\pm0.98$ & $2.87\pm0.10$ &$5.63\pm2.26$ & $1.0\pm2.00$ & $4.0\pm2.61$ & $9.6\pm4.30$ & $0.84\pm0.13$ & $0.78\pm0.13$ & $0.82\pm0.07$ \\
& GOLEM-NV & $5.43\pm1.12$ & $20.67\pm8.74$ & $76.64\pm18.72$ & $7.20\pm2.93$ & $36.80\pm19.33$ & $149.00\pm45.55$ & $0.80\pm0.07$ & $0.62\pm0.19$ & $0.43\pm0.14$ \\
& GraN-DAG & - & - & - & $2.2\pm2.95$ & $9.4\pm4.04$ & $18.6\pm 7.73$ & - & - & -\\
& ICDH(ours) &  $0.55\pm0.98$ & $2.83\pm1.01$ & $5.39\pm2.31$ & $1.0 \pm 2.00$ & $4.0\pm2.61$ & $8.6\pm4.00$ & $0.84\pm0.13$ & $0.78\pm0.13$ & $0.83\pm0.08$ \\
\hline
\end{tabular}}

\caption{Comparison of all baseline algorithms on nonlinear synthetic homoscedastic noise datasets with equal variances across variables (homo-NV): results (mean $\pm$ standard deviation over $10$ trails) on auSHDC, SHD, and auPRC. 
}
\label{tb:homo-nv}
\end{table*}

%% file: tb-hetero.tex
\begin{table*}[hpt]
\centering
\scriptsize
\scalebox{0.77}{
\begin{tabular}{cc|cccc|cccc|cccc}
\hline
\hline
 &&\multicolumn{4}{c}{\underline{auSHDC}} & \multicolumn{4}{c}{\underline{SHD}} & \multicolumn{4}{c}{\underline{auPRC}} \\
graph & methods & $d5$ & $d10$ & $d20$ & $d50$ & $d5$ & $d10$ & $d20$ & $d50$ & $d5$ & $d10$ & $d20$ & $d50$ \\
\hline
& NOTEARS-MLP & $2.75\pm1.51$ & $10.13\pm2.79$ & $27.02\pm9.74$ & - &  $4.6\pm1.51$ & $20.6\pm4.77$ & $54.4\pm27.00$ & $144.1\pm38.0$ &  $0.25\pm0.03$ & $0.25\pm 0.01$ & $0.17\pm 0.02$ & -  \\
& GOLEM-NV & $2.29\pm1.57$ & $13.14\pm3.25$ & $15.23\pm10.87$ & - & $4.2\pm1.99$ & $20.6\pm16.4$ & $54.40\pm24.20$ & - & $0.37\pm0.08$ & $0.21\pm0.11$ & $0.83\pm0.06$ & -\\
& GraN-DAG & - & - & - & - & $6.2\pm1.92$ & $27.4\pm7.99$ & $86.0\pm44.29$ & - & - & - & - & -\\
ER1 & GraN-DAG++ & - & - & - & - & $4.8\pm2.28$ & $22.2\pm16.5$ & $58.4\pm24.79$ & $161.1\pm10.80$ & - & - & - & -\\
& CAM & - & - & - & - & $7.1\pm1.70$ & $21.3\pm5.48$ & $69.2\pm7.64$ & - & - & - & - & - \\
& LiNGAM & - & - & - & - & $7.0\pm1.61$ & $16.8\pm9.04$ & $27.5\pm5.70$ & - & - & - & - & - \\
& GES & - & - & - & - & $4.0\pm1.95$ & $15.9\pm3.59$ & $23.7\pm4.03$ & - & - & - & - & - \\
& GFBS & - & - & - & - & $\bf{3.7\pm1.77}$ & $11.6\pm3.92$ & $25.6\pm4.58$ & - & - & - & - & - \\
& ICDH(ours) & $2.28\pm0.84$ & $8.91\pm2.61$ & $12.10\pm7.98$ & - & $4.0\pm1.58$ & $\bf{11.4\pm4.93}$ & $\bf{23.6\pm9.34}$  & $\bf{134.5\pm23.40}$ & $0.39\pm0.05$ & $0.20\pm0.03$ & $0.20\pm0.01$ & -\\
\hline
 &&\multicolumn{4}{c}{\underline{auSHDC}} & \multicolumn{4}{c}{\underline{SHD}} & \multicolumn{4}{c}{\underline{auPRC}} \\
graph & methods & $d5$ & $d10$ & $d20$ & $d50$ & $d5$ & $d10$ & $d20$ & $d50$ & $d5$ & $d10$ & $d20$ & $d50$ \\
\hline
& NOTEARS-MLP & $3.29\pm2.36$ & $11.93\pm6.83$ & $15.20\pm7.97$ & - & $5.20\pm3.71$ & $22.0\pm5.30$ & $47.60\pm5.90$ & $111.1\pm11.80$ & $0.61\pm0.04$ & $0.27\pm0.11$ & $0.08\pm0.01$ & -\\
& GOLEM-NV & $3.90\pm2.30$ & $20.05\pm10.17$ &  $55.88\pm38.91$ & - &  $6.20\pm2.10$ & $25.60\pm5.94$ & $67.60\pm7.73$ & - & $0.60\pm0.07$ & $0.47\pm0.15$ & $0.13\pm0.16$ & -\\
ER2 & GraN-DAG++ & - & - & - & -& $\bf{3.8\pm2.44}$ & $17.4\pm5.15$ & $42.6\pm4.31$ & - & - & - & - & -\\
& GraN-DAG & - & - & - & - & $\bf{3.8\pm2.83}$ & $19.6\pm 4.65$ &  $77.60\pm10.20$ & $123.1\pm9.60$  & - & - & - & -\\
& CAM & - & - & - & - & $5.2\pm2.71$ & $22.4\pm3.93$ & $74.1\pm7.60$ & - & - & - & - & - \\
& LiNGAM & - & - & - & - & $11.8\pm1.47$ & $24.1\pm3.36$ & $45.5\pm2.96$ & - & - & - & - & - \\
& GES & - & - & - & - & $10.2\pm2.04$ & $21.5\pm3.64$ & $45.8\pm4.98$ & - & - & - & - & - \\
& GFBS & - & - & - & - & $6.9\pm1.45$ & $18.4\pm3.78$ & $44.9\pm4.63$ & - & - & - & - & - \\
& ICDH(ours) & $2.63\pm1.89$ & $8.39\pm5.28$ & $10.10\pm6.25$ & - & $4.2\pm3.06$ & $\bf{15.0\pm3.05}$ & $\bf{40.0\pm4.40}$ & $\bf{102.0\pm30.80}$ & $0.61\pm 0.07$ & $0.21\pm 0.05$ & $0.08\pm0.01$ & -\\
\hline
\end{tabular}}
\caption{Comparison of all baseline algorithms on nonlinear synthetic heteroscedastic noise datasets: results (mean $\pm$ standard deviation over $10$ trails) on auSHDC, SHD, and auPRC. 
}
\label{tb:hetero}
\vspace{-1em}
\end{table*}